\documentclass{new_tlp}
\usepackage{mathptmx}

\usepackage{color}

\usepackage{url}

\usepackage[cmyk]{xcolor}
\definecolor{color1}{cmyk}{0,100,0,0}
\definecolor{color2}{cmyk}{0,60,100,0}
\definecolor{color3}{cmyk}{100,0,0,0}
\definecolor{color4}{cmyk}{0,0,0,100}

\usepackage{pgfplots}
\pgfplotsset{compat=1.14}
\usepackage{tikz}
\usepackage[caption=false]{subfig}

\let\proof\relax

\usepackage{enumitem}

\usepackage{algorithm}
\usepackage{algorithmic}

\usepackage{amssymb}
\usepackage{amsthm}

\theoremstyle{definition}

\newtheorem{definition}{Definition}[section]

\usepackage{thm-restate}

\title[A Generalised Approach for Encoding and Reasoning with Qualitative Theories in ASP]
        {A Generalised Approach for Encoding and Reasoning with Qualitative Theories in Answer Set Programming}
        
   \author[G. Baryannis et al.]
          {GEORGE BARYANNIS, ILIAS TACHMAZIDIS, SOTIRIS BATSAKIS, GRIGORIS ANTONIOU\\
          University of Huddersfield, UK\\
          \email{\{g.bargiannis, i.tachmazidis, s.batsakis, g.antoniou\}@hud.ac.uk}
          \and MARIO ALVIANO\\
             University of Calabria, Italy\\
             \email{alviano@mat.unical.it}
           \and EMMANUEL PAPADAKIS\\
             Center for Spatial Studies, University of California, Santa Barbara, USA\\
             \email{epd@ucsb.edu}}

\jdate{September 2020}
\pubyear{2020}
\pagerange{\pageref{firstpage}--\pageref{lastpage}}
\doi{S1471068401001193}

\begin{document}

\label{firstpage}

\maketitle

  \begin{abstract}
    Qualitative reasoning involves expressing and deriving knowledge based on qualitative terms such as natural language expressions, rather than strict mathematical quantities. Well over 40 qualitative calculi have been proposed so far, mostly in the spatial and temporal domains, with several practical applications such as naval traffic monitoring, warehouse process optimisation and robot manipulation. Even if a number of specialised qualitative reasoning tools have been developed so far, an important barrier to the wider adoption of these tools is that only qualitative reasoning is supported natively, when real-world problems most often require a combination of qualitative and other forms of reasoning. In this work, we propose to overcome this barrier by using ASP as a unifying formalism to tackle problems that require qualitative reasoning in addition to non-qualitative reasoning. A family of ASP encodings is proposed which can handle any qualitative calculus with binary relations. These encodings are experimentally evaluated using a real-world dataset based on a case study of determining optimal coverage of telecommunication antennas, and compared with the performance of two well-known dedicated reasoners. Experimental results show that the proposed encodings outperform one of the two reasoners, but fall behind the other, an acceptable trade-off given the added benefits of handling any type of reasoning as well as the interpretability of logic programs. This paper is under consideration for acceptance in TPLP.
  \end{abstract}



\section{Introduction}


Reasoning with qualitative theories involves abstracting away from mathematical quantities, using natural language expressions such as relations to compare, rather than measure. Qualitative reasoning is motivated by human cognition and leads to results that may be less precise but are more comprehensible. This makes it suitable for cases where understandable interactions and acceptable explanations are prioritised over high precision, or when the latter is not possible because knowledge is imprecise or incomplete~\cite{Wolter2012}.

While qualitative reasoning can be applied in many fields, the most well-researched domains are spatial and temporal ones, with over 40 different qualitative calculi defined and used in many applications, from naval traffic monitoring, to warehouse process optimisation and robot manipulation. A number of toolkits have been developed that support reasoning with these calculi, with the most prominent ones being GQR~\cite{Westphal2009} and SparQ~\cite{Wolter2012}. These are dedicated to qualitative reasoning and optimised for it and do not support reasoning beyond qualitative calculi. However, real-world problems most often require a combination of qualitative and non-qualitative reasoning, for which such toolkits are not fit for purpose.

One solution to support different forms of reasoning would be to split the problem into qualitative and non-qualitative parts and use different reasoners independently. However, this would potentially require combining different representations and handling any compatibility issues that may arise due to different reasoning mechanisms. Instead, we propose to use a single unifying formalism that is capable of handling any type of reasoning, both qualitative and otherwise. Answer Set Programming (ASP) is one such formalism for a number of reasons. ASP allows for solving hard search and optimisation problems, and reasoning with qualitative relations is one such problem. Previous research~\cite{Li2012,Brenton2016,Baryannis2018} has shown that ASP is capable of expressing different qualitative calculi and several different encodings have been proposed with varying levels of performance and applicability. Also, the logic programming nature of ASP means that encodings are human-readable and configurable which is in line with the notion of prioritising comprehensibility that underlies qualitative reasoning.

The main contributions of this paper are the following:
\begin{itemize}
    \item A generalised approach for creating an ASP encoding based on the formalisation of binary qualitative calculi by~\citeN{Dylla2017}, ensuring that the approach is applicable to any such calculus.
    \item Two optimised versions of the generalised encoding that are rooted in algebraic properties that are common across many (but not all) qualitative calculi.
    \item A prototype tool that automatically generates ASP encodings, from input files in formats accepted by dedicated qualitative reasoning toolkits, in order to facilitate the use of the proposed approach by researchers with no prior expertise in ASP.
\end{itemize}

Experimental evaluation quantifies the improvements in execution time and memory of the optimised versions of the generalised encoding, showing that they can handle consistency problems among qualitative relations over 300 different input elements for calculi that support up to 5 relations. The experiments use a real-world case study that requires both qualitative and non-qualitative reasoning, showing the straightforward manner these can be combined under an ASP formalism. Additional experiments investigate the trade-off between performance in terms of execution time and memory consumption and the benefits of ASP (support of any type of reasoning, human-readability and configurability), showing that the second optimised version outperforms SparQ but lags behind GQR.

The rest of this paper is organised as follows. Section~\ref{sec:motivrel} presents a motivating case study in the domain of telecommunication networks and offers a concise analysis of related efforts on qualitative reasoning. Section~\ref{sec:qc} provides a formalisation of binary qualitative calculi, based on which we propose in Section~\ref{sec:systematic} a systematic approach for creating ASP encodings for such calculi. Section~\ref{sec:optim} proposes two optimisations for calculi that support particular algebraic properties, while Section~\ref{sec:impl} discusses the implementation of a converter tool to produce the proposed encodings, as well as a variant encoding that uses custom propagation. Section~\ref{sec:eval} presents and discusses the experimental evaluation of the proposed encodings, while Section~\ref{sec:concl} concludes and points out future research directions.

\section{Motivation and Related Work}
\label{sec:motivrel}

We begin with a case study in telecommunication networks that illustrates the motivation behind introducing an integrated ASP-based approach to qualitative reasoning. Deploying telecommunication networks involves covering specific regions with base stations (antennas) and simultaneously avoiding using the same frequencies for overlapping regions due to interference. This problem setting becomes more important and complex in the case of the emerging 5G technology, compared to previous telecommunication networks. In 5G networks, increased performance can be achieved by means of millimeter waves and this, in turn, calls for a dense network of base stations, especially in urban network deployments~\cite{nordrum2017everything}. This is due to reduced ranges and limitations in coverage (e.g. in presence of obstacles such as buildings). 


The aforementioned setting calls for solving two problems in parallel: topological arrangement of base stations and allocation of frequencies in a way that interference is avoided. The former is qualitative in nature, since it involves determining whether the coverage regions of two nearby antennas overlap and can be addressed through reasoning based on the region connection calculus (RCC). The latter involves reusing frequencies but only for non-overlapping regions and can be mapped to an instance of the graph colouring problem: graph nodes correspond to regions covered by base stations, arcs correspond to connections between overlapping regions and colours represent frequencies, which need to be different for adjacent nodes in the graph. 


When planning such a deployment, especially at large scales, an integrated representation and reasoning approach is necessary to simultaneously solve both a qualitative (spatial topology) problem and a non-qualitative one (graph colouring). In the remainder of this section, we summarise qualitative reasoning approaches that are not specific to a single qualitative calculus and explore to what extent they can address the problem described above. Qualitative reasoning problems are most often represented as constraint satisfaction problems (CSP): qualitative relations are modeled as a set of n-ary (usually binary) constraints over domain variables and consistency of the set is determined by backtracking, successively instantiating constraint variables until all are instantiated or inconsistency is detected~\cite{Renz2007}. The most prominent CSP-based toolkits are GQR~\cite{Westphal2009} and SparQ~\cite{Wolter2012}, and are discussed next. 

GQR relies on the symbolic path consistency algorithm, which successively refines constraints between variables using composition, and, in cases where this is not sufficient to decide consistency, it uses backtracking. GQR also makes use of several heuristics, such as ones based on weight and cardinality~\cite{vanBeek1996}. These allow it to achieve quick solving times for several well-known calculi such as Allen's interval algebra~\cite{Allen1981}. 
SparQ differs from GQR in that it does not restrict itself to constraint-based reasoning using path-consistency and backtracking. It provides conversions from quantitative to qualitative geometric scene descriptions, while also allowing merging different constraint sets. Finally, it supports reasoning about real-value domain variables using techniques of algebraic geometry and neighbourhood-based reasoning, where two relations are considered neighbours when one can be continuously turned into the other without a third one holding in between. SparQ has been used in a number of applications, such as a verification tool for sea navigation software that checks compliance with the International Maritime Organization collision regulations~\cite{Kreutzmann2013} and a representation and reasoning framework for learning robot manipulation tasks~\cite{Wolter2015}.

Instead of directly solving qualitative CSPs,~\citeN{Pham2008} proposed to first map them into propositional satisfiability problems, essentially converting the question of whether a CSP is consistent into the question of whether an interpretation exists that satisfies the corresponding set of Boolean formulas. This transformation allows the use of efficient SAT solvers for qualitative reasoning, with results that are at least comparable to GQR and even improve on it in cases of calculi with large numbers of relations, when combined with divide-and-conquer~\cite{Li2010} or decomposition~\cite{Huang2013} approaches.

The common characteristic of all aforementioned approaches is that they can successfully address the qualitative reasoning aspects of the case study (the RCC problem instance) but would have to resort to external mechanisms to address non-qualitative parts (the graph colouring problem instance). To address the complete case study, one would need a representation and reasoning approach that is not dedicated to a particular form of problems and can instead support both qualitative and non-qualitative reasoning natively. One potential solution could be Description Logics (DL), where support for qualitative reasoning has been implemented on top of standard DL reasoning. For instance,~\citeN{Batsakis2017} also explored several qualitative temporal and spatial representations using OWL properties and SWRL rules, on which reasoning can be performed using any state-of-the-art DL reasoner. However, the scalability of such an approach might not be guaranteed, especially as the number of rules increases substantially.


Motivated by the fact that ASP-based approaches have shown promise for solving difficult combinatorial search problems,~\citeN{Li2012} investigated the potential of qualitative reasoning using ASP encodings of Allen's interval algebra and RCC-8. While ASP solvers can also be used as SAT solvers, the authors chose to preserve logic program encodings since they are much more human-readable and configurable compared to SAT solver formats such as DIMACS. The results showed that the proposed encodings were not as fast as CSP or SAT-based approaches. Subsequent work by~\citeN{Brenton2016} and~\citeN{Baryannis2018} proposed a number of ASP encodings of qualitative calculi that improve on the ones by Li, while \citeN{Izmirlioglu2018} also proposed such encodings but only for the particular case of the cardinal direction calculus. The present work builds on the generalised ASP encoding developed by~\citeN{Baryannis2018}, proposing a family of encodings grounded in a formalisation of qualitative calculi similarly to~\citeN{Dylla2017}, that can be used to represent any binary quantitative calculus and perform reasoning including both qualitative and non-qualitative aspects as required for the motivating case study.

\section{Qualitative Calculi}
\label{sec:qc}

In this section, we provide a formalisation of qualitative calculi, closely following that of~\citeN{Dylla2017}. For the sake of simplicity, definitions are restricted to the binary case, although we acknowledge the existence of a few ternary calculi in literature, most notably the double-cross calculus~\cite{FreksaZim1992}.

\begin{definition}[\bf Binary Partition Scheme]
\label{def:partition2}
Let $\mathcal{U}$ be a universe and $\mathcal{R}$ a finite, nonempty set of binary relations (called \emph{base} relations) which are jointly exhaustive ($\mathcal{U}\times\mathcal{U}=\bigcup_{r \in \mathcal{R}} r$) and pairwise disjoint (JEPD). A \emph{binary partition scheme} is a pair $(\mathcal{U}, \mathcal{R})$, with $\mathcal{R}$ possibly containing the identity relation $id=\{ (u, u) \mid u \in \mathcal{U}\}$.
\end{definition}

\begin{definition} [\bf Binary Qualitative Calculus]
\label{def:qc2}
A \emph{binary qualitative calculus} is a tuple \linebreak $(Rel, Int, \intercal, \diamond)$, where:
\begin{itemize}[leftmargin=.34in]
    \item $Rel$ is a finite, nonempty set of binary relation symbols.
    \item $Int$ is an interpretation $(\mathcal{U}, \varphi, \cdot^{-1}, \circ)$, where
    \begin{itemize}
        \item $\varphi : Rel \rightarrow 2^{\mathcal{U}\times\mathcal{U}}$ is an injective map assigning a binary relation over $\mathcal{U}$ to each binary relation symbol in $Rel$, such that $(\mathcal{U}, \{\varphi(r) \mid r \in Rel\})$ is a binary partition scheme.
        \item $\cdot^{-1}$ is the converse operation.
        \item $\circ$ is the composition operation on binary relations.
    \end{itemize}
    \item $\intercal$ is the converse operation symbol, which represents a map $Rel \rightarrow 2^{Rel}$, such that for all $r \in Rel, \; r^\intercal = \bigcap \{ S \subseteq Rel \mid \varphi(S) \supseteq \varphi(r)^{-1} \}$
    \item $\diamond$ is the composition operation symbol which represents a map $Rel \times Rel \rightarrow 2^{Rel}$ such that for all $r, s \in Rel, \; \diamond(r, s) = \bigcap \{S \subseteq Rel \mid \varphi(S) \supseteq \circ(\varphi(r), \varphi(s))$
\end{itemize}
\end{definition}


For instance, RCC-5~\cite{Randell1992} is defined by 5 binary JEPD relations describing the possible relations between closed regions, which are illustrated in Fig.~\ref{fig:rcc5}:

\begin{itemize}
    \item Disconnected ($DR$ or $DC$): the two regions share no common area.
    \item Partial Overlap ($PO$): the two regions partially overlap
    \item Proper Part ($PP$) and its converse ($PPi$): one region is wholly within the other
    \item Equal ($EQ$): the two regions are identical.
\end{itemize}

There is only one identity relation according to Definition~\ref{def:qc2}, namely $EQ$. There is a converse operation, e.g. $PP^\intercal = PPi$. Finally, there is a composition operation inferring the relation of region \emph{a} to region \emph{c}, if we know the relation of \emph{a} to a region \emph{b} and the relation of \emph{b} to \emph{c}. 

\begin{figure}
\begin{tikzpicture}
\draw (0,3) circle (1cm) node {$a$};
\draw (0,0) circle (1cm) node {$b$};
\draw (2.8,2.2) circle (1cm) node {$a$};
\draw (2.8,0.8) circle (1cm) node {$b$};
\draw (5.6,1.5) circle (0.5cm) node {$a$};
\draw (5.6,1.5) circle (1cm);
\node at (5.6, 0.7) {$b$};
\draw (8.4,1.5) circle (0.5cm) node {$b$};
\draw (8.4,1.5) circle (1cm);
\node at (8.4, 0.7) {$a$};
\draw (11,1.5) circle (1cm) node[anchor=south] {$a$} node[anchor=north] {$b$};
\node at (0, -1.5) {$DR(a,b)$};
\node at (2.8, -1.5) {$PO(a,b)$};
\node at (5.6, -1.5) {$PP(a,b)$};
\node at (8.4, -1.5) {$PPi(a,b)$};
\node at (11, -1.5) {$EQ(a,b)$};
\end{tikzpicture}
\caption{RCC-5 relations.}
\label{fig:rcc5}
\end{figure}
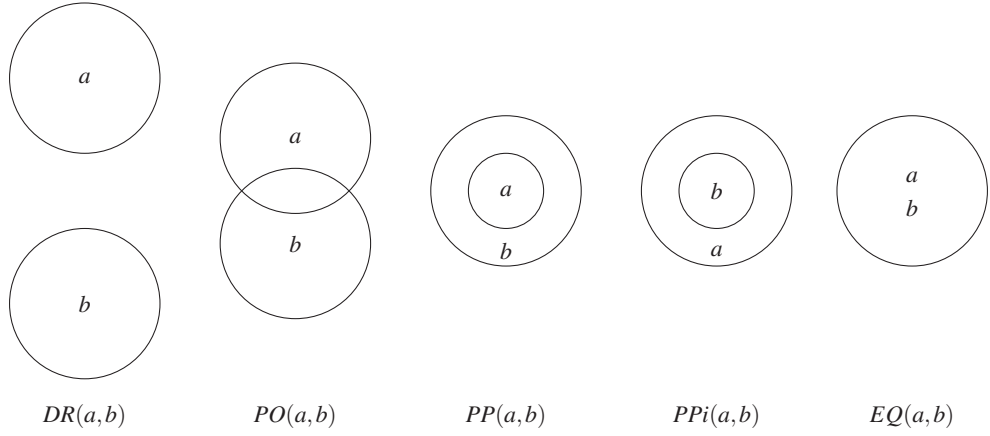

The composition operation of a binary qualitative calculus can be defined through a two-dimensional \emph{composition table} containing the results of the composition of each pair of relations. The full composition table for RCC-5 is provided in Table~\ref{sec:encod}. Using this composition table, we can determine whether a given set of constraints on relations of particular domain elements can actually exist (i.e. is consistent). This problem is formalised in the next definition, following the equivalent definition in~\citeN{Baryannis2018}. Note that in Definition~\ref{def:qc2} the composition and converse operations are of the \emph{weak} variant in~\citeN{Dylla2017}, since all existing binary calculi define at least weak operations.

\begin{table}[ht]
\caption{RCC-5 Composition Table}
\centering
\begin{tabular}{  l  l  l  l  l  l }
	\hline\hline
	Relations & DR & PO & PP & PPi & EQ \\  [0.5ex]
	\hline
	DR & All & DR, PO, PP & DR, PO, PP & DR & DR \\ 
	PO & DR, PO, PPi & All & PO, PP & DR, PO, PPi & PO \\ 
	PP & DR & DR, PO, PP & PP & All & PP\\ 
	PPi & DR, PO, PPi & PO, PPi & EQ, PO, PP, PPi & PPi & PPi \\ 
	EQ & DR & PO & PP & PPi & EQ  \\
    \hline\hline
\end{tabular}
\label{tab:composition_table}
\end{table}

\begin{definition} [\bf Binary Qualitative Calculus Model Existence]
\label{def:problem2}
Let $QC$ be a binary qualitative calculus according to Definition~\ref{def:qc2} and $CT$ the corresponding two-dimensional composition table. Let $V$ be a set of variables ranging over universe $\mathcal{U}$ and $C$ a set of constraint formulas $R(x, y)$ with $x, y \in V$ and $R \subseteq Rel$. Deciding whether a model of $QC$ exists given $C$ is the task to decide whether there is an assignment $\psi: V \rightarrow \mathcal{U}$, such that: (1) $(\psi(x), \psi(y)) \in \varphi(R)$ for all constraints in $C$; (2) if $QC$ includes the identity relation $id$, $(\psi(x), \psi(x)) \in \varphi(id)$; (3) for at least one relation $R_C$ in cell $(R_1, R_2)$ of $CT$, if $(\psi(x), \psi(y)) \in \varphi(R_1)$ and $(\psi(y), \psi(z)) \in \varphi(R_2)$, then $(\psi(x), \psi(z)) \in \varphi(R_C)$.
\end{definition}


\citeN{Dylla2017} provide a detailed analysis of algebraic properties of existing binary qualitative calculi and produce a hierarchy that groups them depending on which properties are satisfied, ranging from Boolean algebras (least amount of properties supported) to full relation algebras (maximum amount of properties supported). We focus here only on two of these properties that have a direct effect on reasoning with ASP and on which we rely to propose the optimisations in Section~\ref{sec:optim}.

\begin{definition} [\bf Involution of converse]
\label{def:r7}
Let $r$ be a binary qualitative relation symbol and $\intercal$ the converse operation symbol. $\intercal$ satisfies the involution property iff $(r^\intercal)^\intercal = r$. 
\end{definition}

Note that the special case of a symmetric relation always satisfies converse involution, since, by definition, a relation symbolised by $r$ is symmetric iff $r^{\intercal} = r$. All relations of RCC-5 satisfy the involution property, since $DR, EQ$ and $PO$ are symmetric and $(PP^\intercal)^\intercal = PPi^\intercal = PP$.

\begin{definition} [\bf Identity Law]
\label{def:r6}
Let $r$ be a binary qualitative relation symbol, $id$ the identity relation and $\diamond$ the binary composition operation symbol as in Definition~\ref{def:qc2}. 
The relation symbolised by $r$ satisfies the identity law iff $r \diamond id = r$. 
\end{definition}

The identity law is relation-dependent, hence a qualitative calculus may include both relations that satisfy it and relations that violate it. All relations in RCC-5 satisfy it, as can be seen in its composition table.

\section{Systematic Approach for Encoding Qualitative Calculi in ASP}
\label{sec:systematic}

In this section, we describe a systematic approach for generating an ASP encoding applicable for any binary qualitative calculus that is described by Definitions~\ref{def:partition2} and~\ref{def:qc2} in the previous section. The approach covers all elements necessary to express model existence problem instances according to Definition~\ref{def:problem2}, namely the domain variables, the contents of the composition table, the way to search for possible models and the constraints that restrict this search.

\subsection{Domain and Base Relations}
\label{sec:domrel}

The first step is to represent the domain, i.e. the elements for which qualitative relations are defined. For this, we define a unary predicate \textit{element}. Any element $x$ that is modelled by the calculus (e.g. regions in RCC), is expressed using a fact $element(x)$. Then, to represent the base relations of the calculus, we define a unary predicate \textit{relation}, which is instantiated for each relation included in the calculus. Relation instantiations can be encoded compactly using term pooling, e.g. for RCC-5 a single fact \verb|relation(dr; eq; po; pp; ppi)| suffices. Instantiating the relation predicate only for the base relations of each calculus enforces the property that relations are jointly exhaustive, according to Definition~\ref{def:partition2}.

\subsection{Composition Table and Search Space}
\label{sec:ctspace}

To encode composition table entries, we first define a \textit{table} predicate with three arguments representing the row relation, column relation, and a valid relation for the composition of the latter two, according to Definition~\ref{def:qc2}. For each cell in the composition table, as many table predicates are instantiated as the possible relations included in that cell. Term pooling can again lead to more compact representation, e.g. for the composition of $DR$ with $PO$ in RCC-5, the following fact is enough: \verb|table(dr, po, (dr;po;pp)).| Collectively, the element, relation and table predicates represent the axiomatic knowledge within a qualitative calculus.

To implement qualitative reasoning for the calculus encoded so far, we first need to ensure that the encoding takes into account that relations are pairwise disjoint, as in Definition~\ref{def:partition2}. We first define a ternary predicate \textit{true}, with $true(X,R,Y)$ denoting that relation R holds for the (ordered) pair (X, Y). Then, we include a choice rule with a conditional literal\footnote{Note that conditional literals are not in ASP-Core-2~\cite{Calimeri2020}, the standard ASP input language format.}: 

\begin{equation}
\label{eq:choice}
\{true(X,R,Y) : relation(R)\} = 1 \leftarrow element(X), element(Y), X != Y.
\end{equation}

which states that for any pair of distinct elements $X$ and $Y$, only one of the instantiated relations can hold. The conditional literal allows for a more convenient and compact representation for conjunctions with variable numbers of literals. Assuming that one of the relations is the identity relation $id$ (as in Definition~\ref{def:partition2}), we also need to encode the case of constraints on a single element, where only the identity relation can hold: $true(X,id,X) \leftarrow element(X)$.

To ensure that any relations that violate the composition table are excluded, we employ the following integrity constraint, again using a conditional literal:

\begin{equation}
\label{eq:ic}
\leftarrow true(X,R_1,Y),\ true(Y,R_2,Z),\ not\ true(X,R_\mathit{out},Z) : table(R_1,R_2,R_\mathit{out}).    
\end{equation}

The semantics of the integrity constraint is that if $R_1$ and $R_2$ hold for $(X,Y)$ and $(Y,Z)$, respectively, then there must be at least one table entry $(R_1,R_2, R_\mathit{out})$ such that $R_{out}$ holds for $(X,Z)$.

\subsection{Input Constraints}
\label{sec:input}


The encoding presented so far is a complete representation of any binary qualitative calculus, defined according to Definitions~\ref{def:partition2} and~\ref{def:qc2}. In order to be able to solve instances of the model existence problem according to Definition~\ref{def:problem2}, we also need to represent input constraints, i.e. relations that may hold among elements according to our knowledge. For this we introduce a ternary predicate \textit{constraint}, with $constraint(X,R,Y)$ denoting that the pair $(X,Y)$ is involved in a constraint, and $R$ is a possible relation for that pair. An additional integrity constraint is required to prevent constraints that violate the composition table:

\begin{equation}
\label{eq:ic2}
\leftarrow constraint(X,\_,Y), not\ true(X,R,Y) : constraint(X,R,Y).
\end{equation}


\subsection{Correctness and Complexity}


The correctness of the generalised encoding that results from applying the presented systematic approach (hereafter referred to as GEN-0) is addressed by the following theorems (proofs are included in~\ref{sec:proofs}).

\begin{restatable}[\bf Soundness]{theorem}{ThmS} 
\label{thm:sound}
Let $Q = \langle QC, C \rangle$ be an instance of the model existence problem according to Definition~\ref{def:problem2}. Let $\Pi_Q$ be the GEN-0 encoding of $Q$ and $S$ be an answer set of $\Pi_Q$. Then, there exists a model $M$ of $Q$ that corresponds to $S$.
\end{restatable}

\begin{restatable}[\bf Completeness]{theorem}{ThmC} 
\label{thm:complete}
Let $Q = \langle QC, C \rangle$ be an instance of the model existence problem according to Definition~\ref{def:problem2} and $M$ a model of $Q$. Let $\Pi_Q$ be the GEN-0 encoding of $Q$. Then, there exists an answer set S of $\Pi_Q$ that corresponds to $M$.
\end{restatable}

Establishing correctness of GEN-0 also has the corollary of identifying the complexity class for the problem of Definition~\ref{def:problem2}. A proof can be found in~\ref{sec:proofs}.

\begin{restatable}[\bf QC Model Existence Complexity]{corollary}{CorC} 
\label{thm:complexity}
The QC model existence problem according to Definition~\ref{def:problem2} is \textsc{NP}-complete.
\end{restatable}

\section{Optimisations based on Algebraic Properties}
\label{sec:optim}


GEN-0 is designed to be applicable to any qualitative calculus, without making any assumptions about additional properties that may hold (such as the ones in Definitions~\ref{def:r7} and~\ref{def:r6}). This allows GEN-0 to be applicable to calculi which correspond to weakly associative or associative Boolean algebras~\cite{Dylla2017}, such as the Rectangle Cardinal Direction (RCD) calculus~\cite{Navarrete2013} and the connected variant of the Cardinal Direction Constraints (CDC) calculus~\cite{Skiadopoulos2005}. While GEN-0 affords maximum applicability, this inevitably results in decreased performance. However, the vast majority of known qualitative calculi are not as restrictive and support one or both of the involution or identity law properties. For these cases, two optimisations are considered in the sequel.

\subsection{Antisymmetric optimisation}

For relations $r$ that satisfy the involution property of Definition~\ref{def:r7}, $r(X,Y) \equiv r^{-1}(Y,X)$. This means that there is no need to consider both pairs $(X,Y)$ and $(Y,X)$ in the choice rule. Only one of them needs to be included, while the other can be generated afterwards, to preserve consistency. This optimisation essentially reduces possible relation pairs by half and is referred to by~\citeN{Brenton2016} as the \emph{antisymmetric optimisation}. This requires the following two changes:

\begin{itemize}
    \item replace $X!=Y$ with $X<Y$ in the choice rule (\ref{eq:choice}) described in Section~\ref{sec:ctspace} and require that all relation predicates included in the predefined rules and facts have a first operand that is arithmetically (or lexicographically) before the second one.
    \item similarly add conjuncts $X<Y$ and $Y<Z$ to the integrity constraint (\ref{eq:ic}) that enforces the composition table, described at the end of Section~\ref{sec:ctspace}.
    \item add rules of the form $true(Y,ri,X) \leftarrow true(X,r,Y), X < Y$, where $ri$ stands for $r^{-1}$, for every pair of converse relations, ensuring that the converse pair is also generated.
\end{itemize}

Note that the latter rule is only required for relations that are not symmetric. In the symmetric case, there is only one relation as the converse of it is the same relation. For calculi where all relations are symmetric (e.g. the TC-6 variant of the Trajectory Calculus~\cite{Baryannis2018}), the choice rule is the only part of the encoding that needs changing to apply the antisymmetric optimisation. The resulting encoding (GEN-1) is suitable for calculi which correspond to any Boolean algebras that support the involution of converse such as the QTC-C variant of the Qualitative Trajectory Calculus~\cite{Weghe2005} or to relation algebras without the identity law, such as the QTC-B variant.

\subsection{Identity optimisation}

The majority of qualitative calculi, as analysed in~\citeN{Dylla2017}, correspond to full relational algebras, satisfying, among others, both the involution of converse property of Definition~\ref{def:r7} and the identity law of Definition~\ref{def:r6}. Calculi in this case include the well-known Allen's interval algebra and all variants of RCC. The identity law allows for a further optimisation, resulting into a third generalised encoding (GEN-2) which requires the following changes:

\begin{itemize}
    \item the integrity constraint (\ref{eq:ic}) in Section~\ref{sec:ctspace} is replaced by the following: $\leftarrow true(X,R_1,Y), \;\\X < Y, \; true(Y,R_2,Z), \; Y < Z, \; R_1 != id, \; R_2 !=id, \; not \ true(X,R_\mathit{out},Z) : table(R_1,R_2,R_\mathit{out})$.
    \item the following two integrity constraints are added: 
    \begin{itemize}
        \item $\leftarrow true(X,id,Y), \; true(Y,R,Z), \; not \ true(X,R,Z), \; Y < Z.$
        \item $\leftarrow true(X,R,Y), \; true(Y,id,Z), \; not \ true(X,R,Z), \; X < Y.$
    \end{itemize}
\end{itemize}

These changes slightly reduce the number of integrity constraints generated based on the composition table, since according to the identity law, composing any relation with the identity relation, always results in the original relation (and no other relations). As discussed in the next section, the integrity constraint that enforces the composition table is the most costly part of the encoding, so any reduction in the table size improves grounding time, especially for large-scale problems. Both the antisymmetric and identity optimisations do not affect correctness, since the correspondence between elements of the model existence problem and elements of the ASP encoding remains identical. Hence, Theorems~\ref{thm:sound} and~\ref{thm:complete} hold for GEN-1 and GEN-2.

Note that there is no known qualitative calculus where the identity law applies but not the involution of converse property. Hence, GEN-2 includes the antisymmetric optimisation. Also, to the best of our knowledge, the remaining properties of relation algebras (such as associativity and distributivity of the composition relation) do not seem capable of contributing to further optimisations. 

\section{Implementation}
\label{sec:impl}


We have explored the applicability of the proposed approach by generating GEN-0, GEN-1 and GEN-2 encodings for several well-known qualitative calculi. To aid in this process and also provide interoperability with existing systems, we have implemented a converter tool that automatically generates ASP encodings from input files in the format accepted by existing toolkits. An initial prototype of the tool currently accepts GQR files as input, containing information about the number of calculus relations, the composition table and a description of identity and converse relations. We intend to support SparQ files as input in a future version. The tool can generate encodings for any binary qualitative calculus (a complete list can be found in~\cite{Dylla2017}). Indicatively, we have generated encodings for Allen's interval algebra, RCC-5 and RCC-8, RCD~\cite{Navarrete2013} and QTC~\cite{Weghe2005}. For the latter two, we had to create GQR input files out of ones available for SparQ, since to the best of our knowledge no GQR implementation for RCD or QTC exists. The resulting encodings are ASP programs that offer a complete representation of a given calculus, which can then be combined with a representation of a particular problem instance (using \textit{element} and \textit{constraint} predicates). All encodings and relevant source code are available at https://github.com/gmparg/ICLP2020.

\subsection{Custom Propagation}
\label{sec:customprop}

The proposed encodings have been implemented with generalisation and applicability in mind, ensuring that they work for any binary qualitative calculus, as shown in the previous section. This inevitably has an effect on performance. A heavy load is imposed on the grounding process due to the rule that generates the integrity constraints that encode the composition table (e.g. the following rule in GEN-2: $\leftarrow true(X,R_1,Y), \; X < Y, \; true(Y,R_2,Z), \; Y < Z, \; R_1 != id, \; R_2 !=id, \; not \ true(X,R_\mathit{out},Z) : table(R_1,R_2,R_\mathit{out})$. All generated constraints have to be grounded for all input elements. This becomes quite complex as input elements increase, especially for calculi with a large number of relations. To reduce grounding effort, we explored custom theory propagation as supported in clingo. Instead of the aforementioned rule, a custom propagator that monitors true/3 literals is implemented. When a true/3 literal changes its truth value, it is checked against the composition table in combination with other such literals. If a violation is found, a nogood is added. The propagator operates on partial assignments according to Algorithm~\ref{alg:prop}. We implemented a Python version using clingo's Python API\footnote{https://potassco.org/clingo/python-api/current/}. Encodings that use the custom propagator are labelled as GEN-3 in the experiments that follow.


\begin{algorithm}[!ht]
\caption{Custom Propagator for Enforcing Composition Table}   \label{alg:prop}
\begin{algorithmic}[1]
\REQUIRE List of grounded true/3 literals $T$, List of grounded table/3 literals $CT$

\FORALL{literals $t \in T$}
    \IF{$t$ not yet assigned truth value}
        \STATE Add $t$ to watchlist $W$
    \ENDIF
\ENDFOR

\FORALL{literals $xy \in W$ of the form $true(X,R_1,Y)$ that change value}
    \FORALL{literals $yz \in T$ of the form $true(Y,R_2,Z)$}
        \IF{$yz$ true}
            \STATE Set flag to false
            \FORALL{literals $r \in CT$ of the form $table(R_1,R_2,R_{out})$}
                \IF{there exists literal $xz \in T$ assigned to true of the form $true(X,R_{out},Z)$}
                    \STATE Set flag to true
                \ENDIF
            \ENDFOR
            \IF{flag still false}
                \STATE Add nogood with the combination of $xy$ and $yz$
            \ENDIF
        \ENDIF
    \ENDFOR
\ENDFOR
\end{algorithmic}
\end{algorithm}

\section{Experiments}
\label{sec:eval}


We conducted two sets of experiments to evaluate the performance of the proposed encodings. The first set, discussed in Section~\ref{sec:expenc}, focuses on quantifying the improvement of the optimisations introduced in GEN-1 and GEN-2, compared to GEN-0, as well as GEN-3 (the version with custom propagation discussed in Section~\ref{sec:customprop}). The second set of experiments, discussed in Section~\ref{sec:expcomp}, focuses on determining how the best performing encoding identified in the first set of experiments fares against dedicated qualitative reasoning systems, specifically GQR and SparQ. This helps quantifying the trade-off between using ASP in order to allow both qualitative and other forms of reasoning, and any effects in performance.

\subsection{Dataset and Processing}

Both sets of experiments were based on the motivating case study described at the beginning of Section~\ref{sec:motivrel}. For locations of antennas, we used the registered antenna structures dataset\footnote{https://hifld-geoplatform.opendata.arcgis.com/datasets/antenna-structure-registrate} extracted from the FCC's Antenna Structure Registration system, which contains 124,811 entries for antennas across USA. In order to reduce the scope of the demonstration to a plausible example, the original dataset was clipped using the official Los Angeles county boundaries\footnote{https://data.ca.gov/dataset/ca-geographic-boundaries}. Each antenna is represented as a geographical feature of point type (latitude, longitude) along with additional properties including unique identifier and the county where it belongs to. The aforementioned data was projected in the World Mercator Projection (EPGS:3857) and then subjected to spatial processing as analysed below.

Assuming a 1000-feet (300m) range for a high-band 5G antenna~\cite{Horwitz2019}, each point in the dataset was converted to a polygon, using a fixed radius buffer. These polygons are identified by the unique identifier of the antenna where they originate from and represent the area of coverage of a corresponding areal transmitter. The generated regions are then combined in a pairwise and exhaustive manner; every pair has its intersection matrix generated (using the DE-9IM 3x3 matrix model~\cite{clementini1995}) along with the representative mask code, which eventually is used to retrieve the RCC-5 base relation that holds between two regions. The following standard mapping from DE-9IM to RCC-5 is used: DR when DE-9IM is disjoint or touches; PO for overlap; PP for covered by; PPi for covers; EQ for identical. The end result of the process is a list of paired regions of coverage associated with an RCC-5 relation. Note that for the sake of simplicity we excluded pairs comprised by a single region or pairs whose inverse counterpart has already been evaluated.


The processed dataset as well as all source code used in the experiments can be accessed at https://github.com/gmparg/ICLP2020. Partial encodings are also included in~\ref{sec:encod}. The software used for the experiments included: (1) the ASP system clingo version 5.4.0~\cite{Gebser2016}, using the implementation that supports Python, necessary for GEN-3; (2) GQR version 1500~\cite{Westphal2009}\footnote{Compiled from https://github.com/m-westphal/gqr}; and (3) the only available SparQ version~\cite{Wolter2012}\footnote{Compiled from https://github.com/dwolter/SparQ}. Grounding times are obtained by running clingo in gringo mode. Time and memory values were calculated using pyrunlim\footnote{https://github.com/alviano/python/tree/master/pyrunlim}. All experiments were performed on a Debian Linux server with an Intel\textregistered~Xeon\textregistered~X3430 CPU at 2.4GHz, with 16 GB RAM.

\subsection{Performance of Encodings}
\label{sec:expenc}

The first set of experiments examines how each of the four versions of the generalised encoding scale by solving the problem of assigning frequencies without interference for an increasing number of antennas, ranging from 10 to 300. A partial input is provided, with only one known relation for each distinct region. The qualitative aspect of the problem is solved by attempting to derive a consistent solution by assigning relations to the remaining region pairs. The non-qualitative aspect of ensuring that no two overlapping regions share the same frequencies is mapped to an instance of the 3-colouring problem, by considering a graph where only overlapping regions are connected. The encoding of the 3-colouring problem in ASP used in the experiments is provided in \ref{sec:encod}.

\begin{figure}
\centering
\subfloat[]{
\begin{tikzpicture}[node distance = 2cm, scale=0.8, transform shape]
\begin{axis}[xlabel= Regions,ylabel= CPU time (s),legend pos= north west]
    \addplot+[color=color1,mark=otimes, mark size=2,error bars/.cd, y dir=both,y explicit] coordinates {
    (10,0.11)
    (20,0.76)
    (30,2.89)
    (40,6.75)
    (50,13.93)
    (60,24.66)
    (70,40.89)
    (80,63.25)
    (90,90.75)
    (100,125.42)
    (110,168.37)
    (120,222.83)
    (130,278.52)    
    (140,356.43)    
    (150,449.06)    
    (160,536.95)
    };
    \addplot+[color=color2,mark=x, mark size=2,error bars/.cd, y dir=both,y explicit] coordinates {
    (10,0)
    (20,0.1)
    (30,0.53)
    (40,1.41)
    (50,2.68)
    (60,4.58)
    (70,7.28)
    (80,10.86)
    (90,15.99)
    (100,20.6)
    (110,28.71)
    (120,37.85)
    (130,50.01)
    (140,58.14)
    (150,73.43)
    (160,89.7)
    (170,114.22)
    (180,130.55)
    (190,159.21)
    (200,173.47)
    (210,207.33)
    (220,242.32)
    (230,275.28)
    (240,322.02)
    (250,357.98)
    (260,408.74)
    (270,473.75)
    (280,495.66)
    (290,549.94)
    (300,646.36)   
    };
    \addplot+[color=color3,mark=oplus, mark size=2, error bars/.cd, y dir=both,y explicit] coordinates {
    (10,0)
    (20,0.1)
    (30,0.54)
    (40,1.19)
    (50,2.46)
    (60,4.37)
    (70,6.76)
    (80,10.35)
    (90,14.96)
    (100,20.1)
    (110,26.69)
    (120,36.84)
    (130,49.02)
    (140,57.16)
    (150,70.37)
    (160,86.66)
    (170,113.18)
    (180,130.55)
    (190,153.05)
    (200,167.4)
    (210,201.17)
    (220,236.23)
    (230,267.99)
    (240,316.99)
    (250,353.98)
    (260,389.1)
    (270,457.18)
    (280,473)
    (290,533.48)
    (300,621.73)
    };
    \addplot+[color=color4,mark=+, mark size=2, error bars/.cd, y dir=both,y explicit] coordinates {
    (10,0)
    (20,0.1)
    (30,0.99)
    (40,0.99)
    (50,4.597)
    (60,9.822)
    (70,18.978)
    (80,38.684)
    (90,53.384)
    (100,39.218)
    (110,71.825)
    (120,536.692)
    (130,308.222)
    (140,193.779)
    (150,368.682)
    (160,334.692)
    };
    \legend{GEN-0, GEN-1, GEN-2, GEN-3}
    \end{axis}
\end{tikzpicture}
\label{fig:eval-gen-time}
}\quad
\subfloat[]{
\begin{tikzpicture}[node distance = 2cm, scale=0.8, transform shape]
\begin{axis}[xlabel= Regions,ylabel= Memory (GB),legend pos= north west]
    \addplot+[color=color1,mark=otimes, mark size=2,error bars/.cd, y dir=both,y explicit] coordinates {
    (10,0)
    (20,0.044)
    (30,0.115)
    (40,0.526)
    (50,0.512)
    (60,0.834)
    (70,1.260)
    (80,1.971)
    (90,2.682)
    (100,3.678)
    (110,5.025)
    (120,6.563)
    (130,8.061)
    (140,9.718)
    (150,12.312)
    (160,14.406) 
    };
    \addplot+[color=color2,mark=x, mark size=2,error bars/.cd, y dir=both,y explicit] coordinates {
    (10,0)
    (20,0.017)
    (30,0.037)
    (40,0.065)
    (50,0.115)
    (60,0.179)
    (70,0.269)
    (80,0.410)
    (90,0.549)
    (100,0.719)
    (110,0.966)
    (120,1.219)
    (130,1.566)
    (140,1.877)
    (150,2.263)
    (160,2.703)
    (170,3.356)
    (180,3.888)
    (190,4.474)
    (200,5.109)
    (210,6.087)
    (220,6.907)
    (230,7.822)
    (240,8.850)
    (250,9.882)
    (260,10.993)
    (270,12.757)
    (280,14.012)
    (290,15.344)    
    (300,16.762)
    };
    \addplot+[color=color3,mark=oplus, mark size=2, error bars/.cd, y dir=both,y explicit] coordinates {
    (10,0)
    (20,0.017)
    (30,0.034)
    (40,0.065)
    (50,0.112)
    (60,0.169)
    (70,0.255)
    (80,0.398)
    (90,0.529)
    (100,0.702)
    (110,0.945)
    (120,1.184)
    (130,1.524)
    (140,1.845)
    (150,2.230)
    (160,2.661)
    (170,3.284)
    (180,3.810)
    (190,4.416)
    (200,5.050)
    (210,6.023)
    (220,6.842)
    (230,7.744)
    (240,8.714)
    (250,9.746)
    (260,10.853)
    (270,12.664)
    (280,13.907)
    (290,15.228)    
    (300,16.640)
    };
    \addplot+[color=color4,mark=+, mark size=2, error bars/.cd, y dir=both,y explicit] coordinates {
    (10,0)
    (20,0.019)
    (30,0.022)
    (40,0.025)
    (50,0.031)
    (60,0.039)
    (70,0.050)
    (80,0.068)
    (90,0.082)
    (100,0.08)
    (110,0.108)
    (120,0.312)
    (130,0.222)
    (140,0.186)
    (150,0.271)
    (160,0.274)
    };
    \legend{GEN-0, GEN-1, GEN-2, GEN-3}
    \end{axis}
\end{tikzpicture}
\label{fig:eval-gen-mem}
}\\
\subfloat[]{
\begin{tikzpicture}[node distance = 2cm, scale=0.79, transform shape]
\begin{axis}[xlabel= Regions,ylabel= Grounding time (s),legend pos= north west]
    \addplot+[color=color1,mark=otimes, mark size=2,error bars/.cd, y dir=both,y explicit] coordinates {
    (10,0.09)
    (20,0.64)
    (30,2.2)
    (40,5.25)
    (50,10.2)
    (60,18.04)
    (70,28.37)
    (80,42.59)
    (90,59.95)
    (100,84.49)
    (110,109.67)
    (120,124.87)
    (130,159.39)    
    (140,202.09)    
    (150,249.61)    
    (160,302.17)
    };
    \addplot+[color=color2,mark=x, mark size=2,error bars/.cd, y dir=both,y explicit] coordinates {
    (10,0)
    (20,0.08)
    (30,0.43)
    (40,0.98)
    (50,1.83)
    (60,3.09)
    (70,4.99)
    (80,7.26)
    (90,10.33)
    (100,14.43)
    (110,19.03)
    (120,24.61)
    (130,31.71)
    (140,39.8)
    (150,47.93)
    (160,58.04)
    (170,70.19)
    (180,83.38)
    (190,99.59)
    (200,114.75)
    (210,131.99)
    (220,152.28)
    (230,172.56)
    (240,197.92)
    (250,221.26)
    (260,251.59)
    (270,285.04)
    (280,318.55)
    (290,352.04)
    (300,387.43)   
    };
    \addplot+[color=color3,mark=oplus, mark size=2, error bars/.cd, y dir=both,y explicit] coordinates {
    (10,0)
    (20,0.09)
    (30,0.32)
    (40,0.87)
    (50,1.62)
    (60,2.88)
    (70,4.79)
    (80,6.74)
    (90,9.8)
    (100,13.39)
    (110,18)
    (120,22.58)
    (130,28.67)
    (140,36.78)
    (150,44.86)
    (160,55)
    (170,65.12)
    (180,78.32)
    (190,91.46)
    (200,107.66)
    (210,124.96)
    (220,142.12)
    (230,163.43)
    (240,185.73)
    (250,207.93)
    (260,235.35)
    (270,262.74)
    (280,291.15)
    (290,323.62)
    (300,357)
    };
    \addplot+[color=color4,mark=+, mark size=2, error bars/.cd, y dir=both,y explicit] coordinates {
    (10,0)
    (20,0)
    (30,0)
    (40,0)
    (50,0.1)
    (60,0.1)
    (70,0.21)
    (80,0.21)
    (90,0.32)
    (100,0.32)
    (110,0.43)
    (120,0.54)
    (130,0.65)
    (140,0.76)
    (150,0.88)
    (160,1)
    };
    \legend{GEN-0, GEN-1, GEN-2, GEN-3}
    \end{axis}
\end{tikzpicture}
\label{fig:eval-gen-gtime}
}\quad
\subfloat[]{
\begin{tikzpicture}[node distance = 2cm, scale=0.79, transform shape]
\begin{axis}[xlabel= Regions,ylabel= Program size (LOC),legend pos= north west]
\addplot+[color=color1,mark=otimes, mark size=2,error bars/.cd, y dir=both,y explicit] coordinates {
    (10,44)
    (20,49)
    (30,54)
    (40,59)
    (50,64)
    (60,69)
    (70,74)
    (80,79)
    (90,84)
    (100,89)
    (110,94)
    (120,99)
    (130,104)    
    (140,109)    
    (150,114)    
    (160,119)
    };
    \addplot+[color=color2,mark=x, mark size=2,error bars/.cd, y dir=both,y explicit] coordinates {
    (10,46)
    (20,51)
    (30,56)
    (40,61)    
    (50,66)
    (60,71)
    (70,76)
    (80,81)
    (90,86)
    (100,91)
    (110,96)
    (120,101)
    (130,106)
    (140,111)
    (150,116)
    (160,121)
    (170,126)
    (180,131)
    (190,136)
    (200,141)
    (210,146)
    (220,151)
    (230,156)
    (240,161)
    (250,166)
    (260,171)
    (270,176)
    (280,181)
    (290,186)
    (300,191)    
    };
    \addplot+[color=color3,mark=oplus, mark size=2, error bars/.cd, y dir=both,y explicit] coordinates {
    (10,48)
    (20,53)
    (30,58)
    (40,63)
    (50,68)
    (60,73)
    (70,78)
    (80,83)
    (90,88)
    (100,93)
    (110,98)
    (120,103)
    (130,108)
    (140,113)
    (150,118)
    (160,123)
    (170,128)
    (180,133)
    (190,138)
    (200,143)
    (210,148)
    (220,153)
    (230,158)
    (240,163)
    (250,168)
    (260,173)
    (270,178)
    (280,183)
    (290,188)    
    (300,193)    
    };
    \addplot+[color=color4,mark=+, mark size=2, error bars/.cd, y dir=both,y explicit] coordinates {
    (10,89)
    (20,94)
    (30,99)
    (40,104)    
    (50,109)
    (60,114)
    (70,119)
    (80,124)
    (90,129)
    (100,134)
    (110,139)
    (120,144)
    (130,149)
    (140,154)
    (150,159)
    (160,164)
    };
    \legend{GEN-0, GEN-1, GEN-2, GEN-3}
    \end{axis}
\end{tikzpicture}
\label{fig:eval-gen-loc}
}
\caption{Performance results for finding consistent solutions when one relation per region is known.}
\label{fig:eval-gen}
\end{figure}
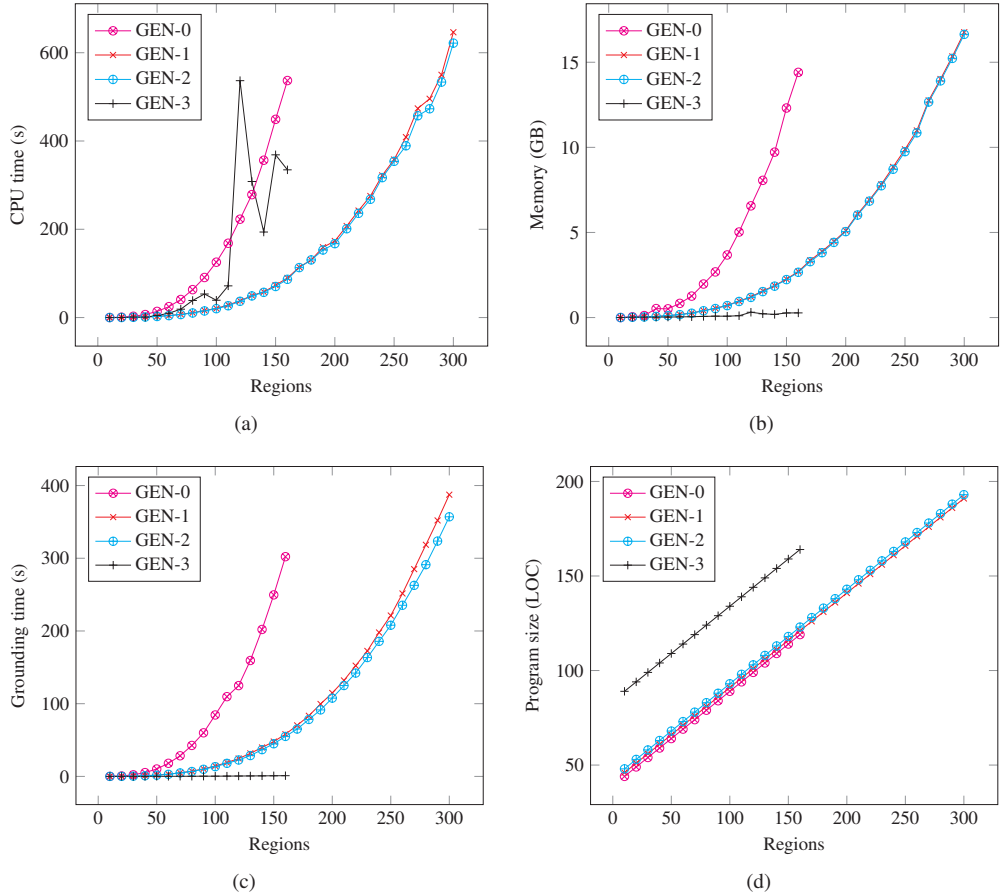


Results are shown in Figure~\ref{fig:eval-gen}, with CPU times for GEN-3 being an average of 10 runs due to variability caused by custom propagation. The antisymmetric and identity optimisations of GEN-1 and GEN-2 allow solutions for up to 300 regions before running out of memory, whereas GEN-0 can only yield results for up to 160 regions. As expected, the improvement of GEN-1 over GEN-0 is much more significant than the improvement of GEN-2 over GEN-1, because the former leads to halving the number of region pairs that are processed whereas the latter only affects the parts of the composition table that include an equality relation. The behaviour of GEN-3 is more complex due to custom propagation: in terms of execution time it is placed lower than GEN-0 but higher than GEN-1 or GEN-2, with a spike around 110 regions. Memory consumption is dramatically lower because of avoiding the expensive grounding of the integrity constraint rule. While this addresses the memory bottleneck, it does not allow GEN-3 to handle larger inputs, because the computational burden is transferred from grounding to solving, leading to CPU times that increase beyond 1000s for more than 160 trajectories.

In terms of grounding time, shown in Figure~\ref{fig:eval-gen-gtime}, the improvement of GEN-1 is still much more significant than that of GEN-2, though the positive effect of GEN-2 is more pronounced here for more than 250 regions. This means that the slight improvement in CPU time of GEN-2 over GEN-1 shown in Figure~\ref{fig:eval-gen-time} is clearly due to reducing grounding time rather than solving time. The comparative minuscule grounding time of GEN-2 is explained similarly to the memory results in Figure~\ref{fig:eval-gen-mem} due to removing the integrity constraint that enforces the composition table.

Program size is compared in Figure~\ref{fig:eval-gen-loc}, where the similarities across GEN-0, GEN-1 and GEN-2 are evident, given that the optimisations introduced in the latter two do not involve significant modifications to the number of rules. The program size of GEN-3 is quite larger due to the inclusion of the Python implementation of the custom propagator, as discussed in Section~\ref{sec:customprop}.

\subsection{Comparison with GQR and SparQ}
\label{sec:expcomp}

The second set of experiments aims to determine how qualitative reasoning using ASP fares against existing state-of-the-art qualitative solvers GQR and SparQ. This can assist in understanding whether there is a tradeoff in terms of performance for opting to use a system that can also handle non-qualitative reasoning. The experiments, hence, involve only the qualitative part of the case study (determining consistency among RCC-5 relations on regions), since GQR and SparQ are unable to handle natively the graph colouring equivalent of the problem of ensuring that no two overlapping regions share the same frequencies. We used the RCC-5 implementations that are provided within the releases of GQR and SparQ. For GQR, we used the consistency command that solves constraint networks, invoked as follows: \verb:./gqr c -C rcc5 -S <input>:. For SparQ, we used the scenario-consistency operation within the constraint-reasoning module: \verb:./sparq constraint-reasoning rcc-5 scenario-consistency first <input>:. Finally, we used the GEN-2 version of the generalised encoding, since it is the best performing one according to the first set of experiments.

Results are shown in Figure~\ref{fig:eval-sota}. SparQ is slower than GEN-2 and consumes similar amounts of memory. Moreover, SparQ is unable to process more than 100 regions, due to heap size limitations of the underlying Lisp interpreter. On the other hand, GQR is clearly more efficient in terms of both CPU time and memory: for 300 regions GQR requires less than 5 seconds and 15MB, two and three orders of magnitude less than GEN-2, respectively. GQR runs for up to around 410 regions, after which execution terminates due to a segmentation fault. These results indicate that while ASP-based encodings improve on SparQ, they are unable to match the performance of GQR, which uses CSP-based reasoning. This should be expected, since GQR is a dedicated qualitative reasoner and, hence, is optimised to solve qualitative problems. However, it is unable to handle any additional non-qualitative aspects, and, hence, would not be able to address problems such as the case study presented in this paper.

\begin{figure}
\centering
\subfloat[]{
\begin{tikzpicture}[node distance = 2cm, scale=0.8, transform shape]
\begin{axis}[xlabel= Regions,ylabel= CPU time (s),legend pos= north west]
    \addplot+[color=color1,mark=otimes, mark size=2,error bars/.cd, y dir=both,y explicit] coordinates {
    (10,0.001)
    (20,0.002)
    (30,0.006)
    (40,0.006)
    (50,0.012)
    (60,0.016)
    (70,0.026)
    (80,0.034)
    (90,0.049)
    (100,0.074)
    (110,0.099)
    (120,0.136)
    (130,0.183)
    (140,0.243)
    (150,0.314)
    (160,0.4)
    (170,0.506)
    (180,0.632)
    (190,0.78)
    (200,0.956)
    (210,1.154)
    (220,1.38)
    (230,1.645)
    (240,1.946)
    (250,2.281)
    (260,2.661)
    (270,3.094)
    (280,3.567)
    (290,4.098)
    (300,4.688)
    };
    \addplot+[color=color2,mark=x, mark size=2,error bars/.cd, y dir=both,y explicit] coordinates {
    (10,0.07)
    (20,0.08)
    (30,0.21)
    (40,0.73)
    (50,1.59)
    (60,3.81)
    (70,7.18)
    (80,13.23)
    (90,22.1)
    (100,35.1)
    };
    \addplot+[color=color3,mark=oplus, mark size=2, error bars/.cd, y dir=both,y explicit] coordinates {
    (10,0)
    (20,0.09)
    (30,0.43)
    (40,0.98)
    (50,2.05)
    (60,3.74)
    (70,5.74)
    (80,8.82)
    (90,12.42)
    (100,17.54)
    (110,23.65)
    (120,31.77)
    (130,40.89)
    (140,50.02)
    (150,62.24)
    (160,76.46)
    (170,94.8)
    (180,112.73)
    (190,129.44)
    (200,150.87)
    (210,179.6)
    (220,206.21)
    (230,234.81)
    (240,265.56)
    (250,301.73)
    (260,341.83)
    (270,388.47)
    (280,428.82)
    (290,476.75)
    (300,555.16)
    };
    \legend{GQR, SparQ, GEN-2}
    \end{axis}
\end{tikzpicture}
\label{fig:eval-sota-time}
}\quad
\subfloat[]{
\begin{tikzpicture}[node distance = 2cm, scale=0.8, transform shape]
\begin{axis}[xlabel= Regions,ylabel= Memory (GB),legend pos= north west]
    \addplot+[color=color1,mark=otimes, mark size=2,error bars/.cd, y dir=both,y explicit] coordinates {
    (10,0)
    (20,0)
    (30,0)
    (40,0)
    (50,0)
    (60,0)
    (70,0)
    (80,0)
    (90,0)
    (100,0)
    (110,0)
    (120,0.0081)
    (130,0.0081)
    (140,0.0084)
    (150,0.0087)
    (160,0.009)
    (170,0.009)
    (180,0.0094)
    (190,0.0094)
    (200,0.0101)
    (210,0.0101)
    (220,0.0103)
    (230,0.0114)
    (240,0.0117)
    (250,0.0118)
    (260,0.0121)
    (270,0.0122)
    (280,0.0136)
    (290,0.0139)    
    (300,0.0142)
    };
    \addplot+[color=color2,mark=x, mark size=2,error bars/.cd, y dir=both,y explicit] coordinates {
    (10,0)
    (10,0.0516)
    (20,0.0517)
    (30,0.0927)
    (40,0.1107)
    (50,0.1397)
    (60,0.1892)
    (70,0.2814)
    (80,0.3724)
    (90,0.4562)
    (100,0.6391)
    };
    \addplot+[color=color3,mark=oplus, mark size=2, error bars/.cd, y dir=both,y explicit] coordinates {
    (10,0)
    (20,0.018)
    (30,0.035)
    (40,0.0601)
    (50,0.1091)
    (60,0.1673)
    (70,0.267)
    (80,0.3516)
    (90,0.5173)
    (100,0.6901)
    (110,0.907)
    (120,1.1632)
    (130,1.5017)
    (140,1.8198)
    (150,2.1998)
    (160,2.6322)
    (170,3.2472)
    (180,3.7707)
    (190,4.3643)
    (200,4.9892)
    (210,5.9609)
    (220,6.7651)
    (230,7.6704)
    (240,8.6317)
    (250,9.6569)
    (260,10.753)
    (270,12.5542)
    (280,13.7995)
    (290,15.1042)    
    (300,16.5059)
    };
    \legend{GQR, SparQ, GEN-2}
    \end{axis}
\end{tikzpicture}
\label{fig:eval-sota-mem}
}
\caption{Performance of the proposed ASP encoding compared to GQR and SparQ.}
\label{fig:eval-sota}
\end{figure}
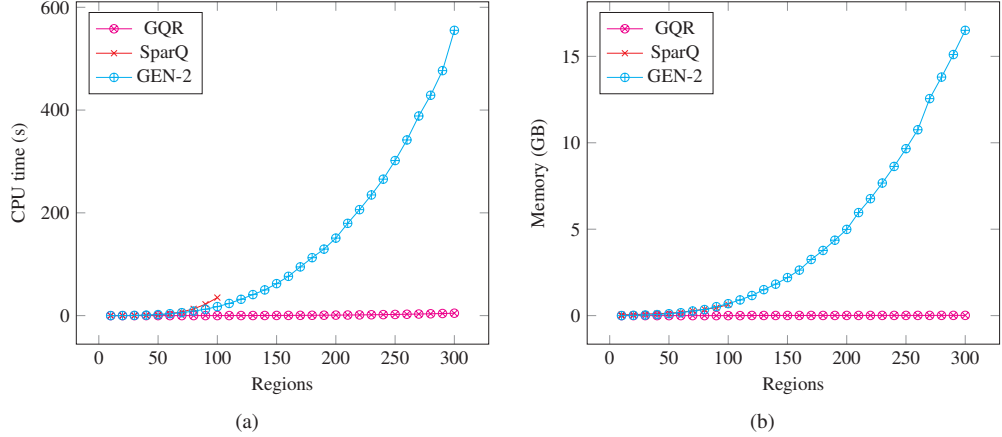

\section{Conclusion and Future Work}
\label{sec:concl}

In this paper, we proposed a generalised approach to encoding qualitative calculi using ASP and a family of ASP encodings that are applicable to any binary qualitative calculus, depending on its algebraic properties. These encodings can be included within any ASP program that also encodes non-qualitative aspects of a problem, such as the presented case study of topological arrangement of base stations and allocation of non-interfering frequencies. Experiments show that the best performing encoding can handle RCC-5 reasoning for up to 300 regions combined with solving an instance of the graph colouring problem. The encodings can be used not only to allow handling of problems that combine qualitative and non-qualitative reasoning but also to extend existing ASP implementations with qualitative aspects.

Future research directions include: (a) exploring additional case studies where the proposed approach can be helpful to address both qualitative and non-qualitative reasoning; (b) expanding the prototype converter tool into a complete toolkit that can facilitate all steps of addressing problems that include qualitative reasoning, from representing and solving them to explaining the produced solutions; (c) examining whether optimisation aspects of the CSP-based implementation of GQR can be used to improve performance of the proposed ASP-based approach.


\clearpage

\appendix

\section{Algorithms and Proofs}
\label{sec:proofs}

\ThmS*
\begin{proof}
Let's assume that no model M of Q exists that corresponds to S, which means there is no such assignment $\psi: V \rightarrow \mathcal{U}$ as described in Definition~\ref{def:problem2}. This would then mean one of the following:
\begin{itemize}[leftmargin=*]
    \item S contains $relation(r)$ for an $r \notin \mathcal{R}$, which contradicts the definition of $relation$ in Section~\ref{sec:domrel}.
    \item S contains $table(r_1,r_2,r_c)$ for $r_1,r_2,r_c \in \mathcal{R}$ with $r_c$ not in cell $(r_1, r_2)$ of CT, which contradicts the definition of $table$ in Section~\ref{sec:ctspace}.
    \item S contains $element(x)$ for an $x \notin V$, which contradicts the definition of $element$ in Section~\ref{sec:domrel}.
    \item S contains $constraint(x,r,y)$ for a constraint $r(x,y) \notin C$, which contradicts the definition of $constraint$ in Section~\ref{sec:input}.
    \item S contains $true(x,r,y)$ when $(\psi(x), \psi(y)) \notin \varphi(R)$, which contradicts the definition of $true$ and related integrity constraints in Section~\ref{sec:ctspace}.
\end{itemize}
Hence, the initial assumption is incorrect, so there exists a model M of Q that corresponds to S.
\end{proof}

\ThmC*
\begin{proof}
Model M corresponds to an assignment $\psi: V \rightarrow \mathcal{U}$ as described in Definition~\ref{def:problem2}. We construct an answer set $S$ of $\Pi_Q$ comprising the following atoms:
\begin{itemize}[leftmargin=*]
    \item $relation(r)$ for all $r \in \mathcal{R}$.
    \item $element(x)$ for all $x \in V$.
    \item $constraint(x,r,y)$ for all constraints $r(x,y) \in C$ with $r \in \mathcal{R}$ and $x,y \in V$.
    \item $true(x,id,x)$ for all $x \in V$.
    \item $table(r_1,r_2,r_c)$ for $r_1,r_2,r_c \in \mathcal{R}$ for all $r_c$ in cell $(r_1, r_2)$ of CT.
    \item $true(x,r,y)$ when $(\psi(x), \psi(y)) \in \phi(R)$.
\end{itemize}
We observe that all rules and constraints of $\Pi_Q$ are satisfied by $S$, and any proper subset of $S$ violates either a fact or one of the two rules defining predicate $true$.
Hence, $S$ is an answer set of $\Pi_Q$.
\end{proof}

\CorC*
\begin{proof}
From Theorems~\ref{thm:sound} and~\ref{thm:complete} and by the fact that answer set existence of ASP programs containing only choice rules and integrity constraints belongs to NP~\cite{DBLP:conf/slp/MarekT89,DBLP:journals/jlp/CadoliS93}, we can deduce that QC model existence also belongs to NP. Additionally, at least one instance of the QC model existence problem, the case of RCC-8, is \textsc{NP}-hard~\cite{Renz1999}. Hence, QC model existence is \textsc{NP}-complete.
\end{proof}

\section{Partial Encodings}
\label{sec:encod}

The partial GEN-0 encoding below only encodes the first row of Table~\ref{tab:composition_table}. Complete versions of the encodings are available at https://github.com/gmparg/ICLP2020.

\begin{verbatim}
{true(X,R,Y) : relation(R)} = 1 :- element(X); element(Y); X != Y.
:- true(X,R1,Y); true(Y,R2,Z); not true(X,Rout,Z) : table(R1,R2,Rout).
true(X,eq,X) :- element(X).
:- constraint(X,_,Y); not true(X,R,Y) : constraint(X,R,Y).
relation(dr; eq; po; pp; ppi).
table(dr, eq, (dr)).
table(dr, po, (dr;po;pp)).
table(dr, pp, (dr;po;pp)).
table(dr, ppi, (dr)).
table(dr, dr, (eq;po;pp;ppi;dr)).
\end{verbatim}

The GEN-1 encoding replaces the first two lines above with the following:

\begin{verbatim}
{true(X,R,Y) : relation(R)} = 1 :- element(X); element(Y); X < Y.
:- true(X,R1,Y); X < Y; true(Y,R2,Z); Y < Z; 
     not true(X,Rout,Z) : table(R1,R2,Rout).
true(Y,ppi,X) :- true(X,pp,Y), X < Y.
true(Y,pp,X) :- true(X,ppi,Y), X < Y.
\end{verbatim}

The latter two lines are to ensure converse pairs are produced for completeness but can be omitted if not required. The GEN-2 encoding replaces the second line above with the following:

\begin{verbatim}
:- true(X,eq,Y); true(Y,R,Z); not true(X,R,Z); Y < Z.
:- true(X,R,Y); true(Y,eq,Z); not true(X,R,Z); X < Y.
:- true(X,R1,Y); X < Y; true(Y,R2,Z); Y < Z; R1!=eq; R2!=eq; 
                     not true(X,Rout,Z) : table(R1,R2,Rout).
\end{verbatim}

The GEN-3 encoding removes these lines completely, since the composition table is enforced through custom propagation using Python code.

In the experiments the following encoding of the 3-colouring problem is used:

\begin{verbatim}
color(red; green; blue).
{hasColor(X,C) : color(C)} = 1 :- element(X).
:- arc(V1, V2), hasColor(V1, X), hasColor(V2, Y), X=Y.
arc(V2, V1):- arc(V1, V2).
arc(V1, V2):-true(V1,eq,V2), V1!=V2.
arc(V1, V2):-true(V1,po,V2).
arc(V1, V2):-true(V1,pp,V2).
arc(V1, V2):-true(V1,ppc,V2).
\end{verbatim}

\label{lastpage}
\end{document}